\documentclass[lettersize,journal]{IEEEtran}
\usepackage{booktabs}
\usepackage{amsmath}
\usepackage{balance}
\usepackage{enumitem}
\usepackage{amsthm}
\usepackage{hyperref}
\usepackage{amssymb}
\usepackage{colortbl}
\usepackage{stfloats}
\usepackage{xcolor} 
\usepackage{textcomp}
\usepackage{graphicx}
\usepackage{verbatim}
\usepackage{multirow}
\usepackage{algorithm}
\usepackage{multicol}
\usepackage{algpseudocode}
\usepackage{mathtools}
\newtheorem{theorem}{Theorem}
\newtheorem{lemma}[theorem]{Lemma}

\usepackage{cite}
\begin{document}

\title{Enhancing Federated Learning Through Secure Cluster-Weighted Client Aggregation}

\author{Kanishka Ranaweera,~\IEEEmembership{Student Member, ~IEEE, }  Azadeh Ghari Neiat,~\IEEEmembership{Member, ~IEEE, }, Xiao Liu,~\IEEEmembership{Senior Member, ~IEEE, }, Bipasha Kashyap,~\IEEEmembership{Member, ~IEEE, },Pubudu N. Pathirana,~\IEEEmembership{Senior Member, ~IEEE, }

\IEEEcompsocitemizethanks{\IEEEcompsocthanksitem Kanishka Ranaweera is with School of Engineering and Built Environment, Deakin University, Waurn
Ponds, VIC 3216, Australia. \protect\\
E-mail: k.ranaweera@deakin.edu.au
\IEEEcompsocthanksitem Azadeh Ghari Neiat is with the School of Electrical Engineering and Computer Science, The University of Queensland, St Lucia, QLD 4072, Australia. \protect\\
E-mail: a.gharineiat@uq.edu.au;
\IEEEcompsocthanksitem  Xiao Liu is with the School of Information Technology,
Deakin University, Waurn Ponds, VIC 3216, Australia. \protect\\
E-mail: xiao.liu@deakin.edu.au;
\IEEEcompsocthanksitem Bipasha Kashyap is with the School of Engineering and Built Environment, Deakin University, Waurn Ponds, VIC 3216, Australia. \protect\\
E-mail: b.kashyap@deakin.edu.au
\IEEEcompsocthanksitem Pubudu N. Pathirana is with School of Engineering and Built Environment, Deakin University, Waurn Ponds, VIC 3216, Australia. \protect\\
E-mail: pubudu.pathirana@deakin.edu.au

}
}


\maketitle

\begin{abstract}
Federated learning (FL) has emerged as a promising paradigm in machine learning, enabling collaborative model training across decentralized devices without the need for raw data sharing. In FL, a global model is trained iteratively on local datasets residing on individual devices, each contributing to the model's improvement. However, the heterogeneous nature of these local datasets, stemming from diverse user behaviours, device capabilities, and data distributions, poses a significant challenge. The inherent heterogeneity in federated learning gives rise to various issues, including model performance discrepancies, convergence challenges, and potential privacy concerns. As the global model progresses through rounds of training, the disparities in local data quality and quantity can impede the overall effectiveness of federated learning systems. Moreover, maintaining fairness and privacy across diverse user groups becomes a paramount concern.
To address this issue, this paper introduces a novel FL framework, ClusterGuardFL,  that employs dissimilarity scores, k-means clustering, and reconciliation confidence scores to dynamically assign weights to client updates. The dissimilarity scores between global and local models guide the formation of clusters, with cluster size influencing the weight allocation. Within each cluster, a reconciliation confidence score is calculated for individual data points, and a softmax layer generates customized weights for clients. These weights are utilized in the aggregation process, enhancing the model's robustness and privacy. Experimental results demonstrate the efficacy of the proposed approach in achieving improved model performance in diverse datasets.
\end{abstract}


\section{Introduction}
\IEEEPARstart{I}{n} the dynamic landscape of technology, the Internet of Things (IoT) has emerged as a powerful force, weaving together devices, sensors, and systems to enable seamless communication and data exchange. This interconnected ecosystem, however, presents unique challenges, particularly when it comes to optimizing machine learning models for diverse and decentralized data sources\cite{al2015internet,ma2019high,mohanta2020survey}.
However, the traditional approach to machine learning grapples with a significant hurdle—centralized model training. In conventional settings, data is often consolidated in a central repository for model refinement. This centralized paradigm, while effective in some scenarios, raises concerns related to privacy, latency, and bandwidth. Moreover, it encounters limitations when applied to the expansive and varied world of IoT, where data is generated and stored across a multitude of devices\cite{li2018learning}.

Federated learning (FL) has been propelled to the forefront as a revolutionary approach to address this challenge \cite{mcmahan2017communication}. In FL, a shared machine learning model is collaboratively trained across heterogeneous decentralised devices while keeping all the training data on the device locally. This optimizes the learning process across the varied and dispersed data inherent in IoT environments, while, at the same time, protecting privacy and security, safeguarding against the centralization of sensitive information and saving data centres from the significant burden of storing and processing such a huge amount
of data. 

Despite its advantages, the decentralized nature of FL presents unique challenges. Key among these is the issue of uneven data distribution across devices. Such imbalances can result in biased global model updates, thereby affecting the overall \textbf{fairness} of the system.  Additionally, FL's open structure makes it particularly vulnerable to \textbf{data poisoning} attacks, characterized by deliberate manipulation of training data to compromise the integrity of the machine learning model. In an IoT context, where devices often operate autonomously and interact with unpredictable environments, the risk of such attacks is magnified. Malicious actors can exploit the decentralized nature of FL to introduce corrupted data or inject adversarial updates into the model training process, aiming to degrade performance or cause erroneous outcomes. To fully leverage the potential of FL in real-world applications, it is essential to implement comprehensive detection and countermeasures. Such strategies are vital in mitigating the risks and ensuring FL systems operate effectively and fairly.

\subsection{Related Work}

Previous research has dedicated considerable effort to devising methodologies that effectively address fairness challenges in FL, particularly in scenarios involving non-Independently and Identically Distributed (non-IID) data among clients. Approaches such as weighted aggregation \cite{ezzeldin2023fairfed, reddi2020adaptive, wang2020tackling} and personalized FL \cite{fallah2020personalized, deng2020adaptive} have emerged as promising solutions, aiming to balance the significance of individual client contributions. These frameworks mitigate biases introduced during model updates, promoting equitable representation across diverse user groups. However, despite their success in addressing fairness, these methodologies often assume that all client participants are honest, which may not always be the case. This assumption can inadvertently leave FL systems vulnerable to data poisoning attacks, where malicious clients inject manipulated or poisoned data to degrade the global model's performance. 

Data poisoning attacks, such as label-flipping \cite{shafahi2018poison}, backdoor attacks \cite{bagdasaryan2020backdoor}, and more recent local model poisoning attacks \cite{baruch2019little, fang2020local, shejwalkar2021manipulating}, present significant challenges to FL systems. These attacks manipulate training data or model updates in diverse ways, compromising the integrity of the learning process. While existing Byzantine-robust defenses are effective at detecting straightforward manipulations \cite{blanchard2017machine, cao2020fltrust, fung2018mitigating, yin2018byzantine}, the increasing sophistication of local model poisoning attacks highlights a critical gap in the current state of the art.

One of the most widely adopted aggregation methods in FL is \textbf{FedAvg} \cite{mcmahan2017communication}, which computes a weighted average of client updates. Although FedAvg is simple and effective, it is particularly susceptible to Byzantine attacks, where malicious clients can corrupt the global model by introducing false updates. To counteract such vulnerabilities, several robust aggregation frameworks have been proposed.

\textbf{Median} and \textbf{Trimmed Mean}~\cite{yin2018byzantine} are two such approaches that enhance robustness against outliers. The Median method computes the median of client updates, thus providing protection against extreme values but at the risk of discarding useful information from benign clients. Similarly, the Trimmed Mean approach calculates the mean after trimming a certain percentage of extreme values, striking a balance between robustness and information retention.

More advanced methods, such as \textbf{Krum} \cite{blanchard2017machine}, select the client update closest to the majority of other updates, effectively filtering out anomalous updates but at a higher computational cost. \textbf{GeoMed} \cite{chen2017distributed}, or Geometric Median aggregation, finds the point that minimizes the sum of distances to all client updates, offering robustness but with slower convergence. \textbf{AutoGM} \cite{li2021byzantine} introduces automation in geometric median-based aggregation, dynamically adjusting to improve robustness against malicious updates. Additionally, \textbf{Clustering}-based \cite{sattler2020byzantine} approaches group client updates into clusters before aggregation, thereby reducing the influence of outliers and malicious clients.

While these methods have contributed significantly to improving the robustness and security of FL systems, they often struggle to achieve an optimal balance between accuracy, fairness, and security, especially in the presence of adversarial attacks.  Table 1 summarizes the trade-offs across these prominent aggregation frameworks, highlighting their performance with respect to accuracy, fairness, and security, underscoring the need for continuous innovation to achieve higher accuracy, fairness, and security in FL systems.

\begin{table}[t!]
\centering
\caption{Comparison of Related Work}
\label{tab:related_work}
\begin{tabular}{|l|c|c|c|}
\hline
\textbf{Method} & \textbf{Accuracy} & \textbf{Fairness} & \textbf{Security} \\ \hline
FedAvg \cite{mcmahan2017communication} & Moderate & Low & Low \\ \hline
Median \cite{yin2018byzantine} & Low-Moderate & Moderate & Moderate \\ \hline
Trimmed Mean \cite{yin2018byzantine} & Moderate & Moderate & Moderate \\ \hline
Krum \cite{blanchard2017machine} & Moderate-High & Low-Moderate & High \\ \hline
GeoMed \cite{chen2017distributed}& Moderate & Low-Moderate & High \\ \hline
AutoGM \cite{li2021byzantine} & High & Low-Moderate & High \\ \hline
Clustering \cite{sattler2020byzantine}& Moderate & Moderate & Moderate \\ \hline
\textbf{ClusterGuardFL} & \textbf{High} & \textbf{High} & \textbf{High} \\ \hline
\end{tabular}
\end{table}

\subsection{Motivation}

This work is motivated by the dual challenges of advancing fairness and fortifying FL systems against data poisoning attacks. Although previous research has advanced fairness through methods like weighted aggregation and personalized FL, there persists a significant vulnerability in FL systems concerning their robustness against adversarial manipulations. The common assumption that all client participants are honest, coupled with a primary focus on fairness, might overlook the risks posed by compromised clients introducing manipulated data. To address this, we introduce ClusterGuardFL, a novel framework that dynamically clusters clients based on dissimilarity scores. This method not only enhances fairness but also robustly shields FL systems from the threats of evolving data poisoning attacks. By concurrently addressing these key challenges, ClusterGuardFL aims to contribute to the development of more secure and equitable FL systems, suitable for a range of real-world applications.

Our paper provides the following key contributions:

\begin{enumerate}
    \item We present ClusterGuardFL, an innovative FL framework tailored to address challenges in fairness and security against data poisoning attacks by leveraging clustering for client aggregation.
    \item We provide rigorous analytical proofs to establish the theoretical foundation of ClusterGuardFL, demonstrating its efficacy and reliability in FL scenarios
    \item We conduct a holistic convergence analysis, offering insights into the stability and efficiency of the ClusterGuardFL methodology throughout the iterative model training process.
    \item Through extensive experimentation with diverse real-world datasets, we empirically validate the effectiveness of ClusterGuardFL, showcasing its impact on improving model performance, preserving privacy, and achieving convergence in real-world scenarios.
\end{enumerate}

The paper is organized into four sections to provide a clear and comprehensive exploration of ClusterGuardFL’s theoretical framework, practical implementation, empirical validation, and its implications as follows.  Section II explores foundational concepts related to our ClusterGuardFL framework. Section III introduces ClusterGuardFL and details its methodologies and theoretical foundations. Section IV presents results from three distinctive dataset experiments and showcases the effectiveness of the ClusterGuardFL framework. Finally, Section V synthesizes key findings and discusses contributions, and suggests future research directions. 

\section{Preliminaries}

FL marks a transformative shift in machine learning paradigms since its inception in 2016 \cite{mcmahan2017communication}. This innovative paradigm addresses the challenges of privacy and data decentralization by enabling collaborative model training across a network of distributed devices, such as mobile phones, IoT devices, and edge servers. Unlike traditional machine learning models that require centralized data aggregation, FL decentralizes the model training process, allowing devices to train models locally without sharing raw data. This method significantly enhances data privacy and reduces the risks associated with data transfer.

The FL workflow unfolds over multiple iterations or rounds, with each participating device computing a model update independently using its local data. These updates are then aggregated typically by a central server to develop a global model that reflects insights from the entire network. The inherently decentralized nature of FL not only preserves user privacy by retaining sensitive data on local devices but also scales effectively to accommodate large, distributed environments.

One notable feature of FL is its adaptability to diverse applications. For example, in healthcare, FL enables medical institutions to collaboratively develop predictive models while complying with strict privacy regulations, without exchanging patient data. Similarly, in areas like personalized recommendations and edge computing, FL proves invaluable by allowing data to remain where it is generated, thus optimizing processing speeds and responsiveness.

In mathematical terms, we assume there are \( K \) active clients, such as mobile phones, computers, or institutions like medical centers and banks, each with its own local dataset drawn from a distribution \( D_k \). The goal is to collaboratively solve an empirical risk minimization problem across all clients, where the total sample size is \( n = \sum_{k=1}^K |D_k| \).

During each communication round, every client trains a local model on its own dataset and computes an updated model \( w_k \) by minimizing a local loss function \( F_k(w) \), which captures the error specific to its data:

\[ w^*_k = \text{argmin} \ F(w_k), \ k \in K \ \text{(1)} \]

where $F(w_k)$ is specified through a loss function dependent on an input-output data pair $(x_i, y_i)$. Typically, $x_i \in \mathbb{R}^d$ and $y_i \in \mathbb{R}$ or $y_i \in \{-1, 1\}$. Examples of such popular loss functions include: cross-entropy loss \cite{cox1958regression} and mean squared error\cite{lehmann2006theory}

Each device \(k\) uploads its computed update \(w_k\) to the MEC server, which then aggregates and calculates a new version of the global model as:

\[ w_G = \frac{1}{\sum_{k \in K} |D_k|} \sum_{k \in K} |D_k|w_k \ \text{(2)} \]

This global model is then downloaded to all devices for the next round of training until the global learning is complete. Notably, in classical FL, the global model \(w_G\) is computed and updated by the server. 

The server facilitates the joint iterative distributed training as follows:

\begin{enumerate}[topsep=0pt, itemsep=-1ex, partopsep=1ex, parsep=1ex]
    \item \textbf{Initialization:} The server initiates by randomly selecting or pretraining a global model using public data.
    \item \textbf{Broadcasting:} The server shares the global model's parameters with the clients and their local models are updated accordingly.
    \item \textbf{Local Training:} Each client independently updates its local model using its specific data and uploads the modified model parameters back to the server.
    \item \textbf{Aggregation:} The central server aggregates the received models by applying weights and transmits the updated models back to the individual nodes.
    \item \textbf{Termination:} All phases, except for the \textit{Initialization} step, is repeated until a predefined termination condition is satisfied (e.g., achieving a specific model accuracy or reaching a maximum number of iterations).
\end{enumerate}

\begin{figure*}[!t]
\centering
\includegraphics[width=\textwidth]{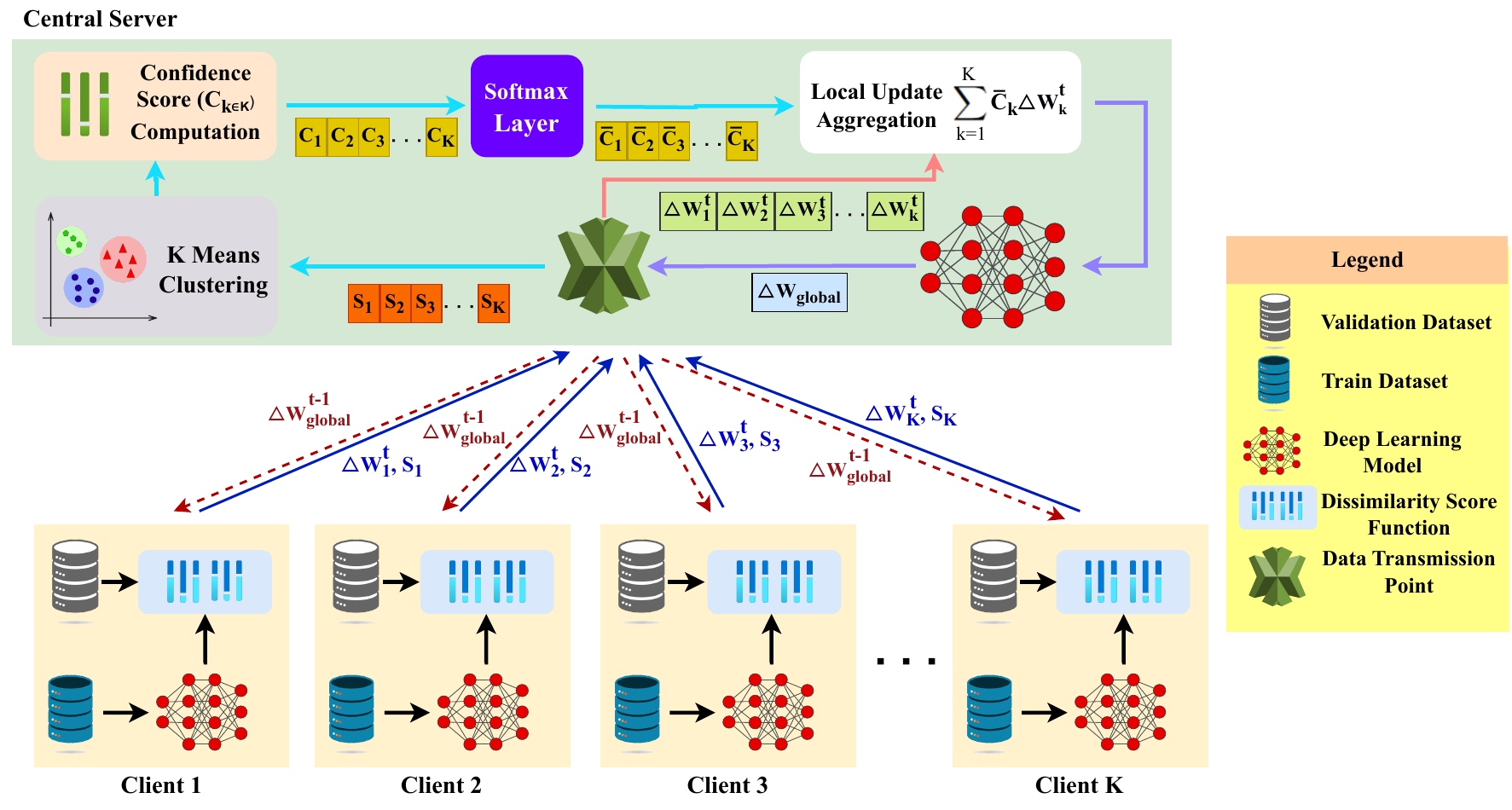}
\caption{Proposed ClusterGuard architecture}
\label{fig_1}
\end{figure*}

\section{ClusterGuardFL Framework}

\subsection{Threat Model}

In this subsection, we delineate the threat model underlying our proposed FL framework, ClusterGuardFL. The primary focus is on mitigating the risks posed by adversarial clients through data poisoning and model poisoning attacks. Our threat model includes the following key considerations:

\subsubsection{Adversarial Clients}
We assume the presence of adversarial clients within the FL network. These clients have the capability to manipulate their local datasets or inject malicious updates during the model training process. Their primary goal is to bias the global model, thereby compromising its integrity and fairness.

\subsubsection{Data Poisoning Attacks}
Data poisoning attacks involve adversarial clients strategically injecting poisoned data into their local datasets. This malicious injection aims to influence the learning process during the global model aggregation, potentially compromising the fairness and accuracy of the global model. By manipulating the training data, adversarial clients can degrade the model's performance and reliability.

\subsubsection{Model Poisoning Attacks}
Model poisoning attacks occur when adversarial clients inject malicious updates directly into the model during aggregation rounds. These attacks aim to introduce biases or vulnerabilities into the global model, adversely affecting its overall performance and robustness. Such attacks can be more subtle and harder to detect compared to data poisoning attacks.

\textbf{Assumptions}
To focus our defense mechanisms effectively, we make certain assumptions:
\begin{itemize}
    \item Adversaries do not have the ability to eavesdrop on communications between clients and the central server.
    \item The central server is considered an honest entity and is free from malicious intent.
    \item The primary objective of adversarial clients is to bias the global model rather than to disrupt the overall communication or operation of the FL system.
\end{itemize}

These assumptions allow us to concentrate our defense strategies specifically on mitigating the impact of adversarial behaviors within client devices. However, real-world scenarios often do not adhere to these assumptions. To address this, several studies have focused on more complex threat models. For instance, Wei et al. \cite{wei2020framework} provides a framework to evaluate gradient leakage attacks, where adversaries can reconstruct private training data by intercepting gradient updates from clients, highlighting the potential vulnerabilities even when assuming secure communication channels. Additionally, the work by Xie et al. \cite{xie2022securing} proposes using covert communication techniques to prevent adversaries from intercepting model updates, adding another layer of security to the FL process. These approaches provide robust defenses in scenarios where stronger assumptions about the trustworthiness of the server or communication infrastructure do not hold.

By addressing these threats, our ClusterGuardFL framework aims to enhance the security and fairness of FL systems, ensuring robust and reliable model performance in the presence of adversarial clients.

\subsection{ClusterGuardFL Model}

In this section, we present a comprehensive overview of our novel framework designed to augment fairness and security in FL. Our method unfolds through a series of intricate steps, combining dissimilarity score calculation, K-means clustering, reconciliation confidence scoring, and Softmax-weighted aggregation to dynamically regulate client contributions during the global model aggregation process.

\subsubsection{Calculation of Dissimilarity Score}
Our framework commences by calculating the dissimilarity scores between the predictions of the global model and the updates from individual local models. The Earth Movers Distance (EMD) serves as a robust metric to quantify the distribution mismatch between these models. This dissimilarity score computation provides a nuanced understanding of the divergence between global and local models, offering valuable insights for subsequent adjustments.

\subsubsection{K-Means Clustering}

K-means clustering is a widely used algorithm across various fields, including image segmentation, customer segmentation, and anomaly detection \cite{dhanachandra2015image,kansal2018customer,munz2007traffic}. Its simplicity, efficiency, and effectiveness make it an ideal choice for identifying and exploring the underlying structure within datasets.

In our framework, once the participating clients submit their dissimilarity scores, the global server initiates the K-means clustering process. This process not only considers the dissimilarity scores but also takes into account the sizes of the clients' local datasets. By clustering clients with similar dissimilarity scores and dataset sizes, we establish a more nuanced method for detecting patterns in model disparities across the network.

\subsubsection{Reconciliation Confidence Score}
The ensuing step involves the calculation of a reconciliation confidence score for each data point within the identified clusters. This score is derived by evaluating the distance of each data point to the centroid of its respective cluster. Crucially, the confidence score is augmented by the size of the cluster, ensuring that larger clusters and data points closer to the cluster centroids command higher confidence. This innovation provides a nuanced measure that reflects both the significance of clusters and the proximity of data points to their respective cluster centers.

\subsubsection{Softmax-weighted Aggregation}
The generated reconciliation confidence scores undergo a Softmax layer, transforming them into dynamic weights. These weights encapsulate the reliability and influence of each client in the FL process. Clients with higher reconciliation confidence receive larger weights, creating a personalized and adaptive weighting scheme that reflects their individual contributions to the collaborative learning process.

\subsubsection{Global Aggregation Protocol}
In the final stage, the Softmax-weighted reconciliation confidence scores are seamlessly integrated into the global aggregation protocol. During this phase, model updates transmitted by individual clients are aggregated based on their corresponding weights. This dynamic weighting mechanism ensures that clients with higher confidence play a more substantial role in shaping the updated global model. The resultant updated global model is then disseminated back to the participating clients, closing the iterative cycle of our FL framework.

To provide a more detailed understanding of our framework, refer to Algorithm 1, which outlines the pseudocode for our framework. Additionally, Fig.~\ref{fig_1} illustrates the proposed architecture, offering a visual representation of the components and flow of information in our framework.

Our framework stands as a strategic fusion of dissimilarity metrics, clustering, and confidence scoring, offering a holistic solution to address fairness and security concerns in FL. This adaptive methodology aims to foster equitable and robust collaborative learning environments by dynamically adjusting the influence of individual clients based on dissimilarity patterns and dataset sizes. Further elaborations, analyses, and empirical validations are provided in the subsequent sections to underscore the efficacy of our proposed framework.

\begin{algorithm}
\label{alg1}
  \caption{Pseudocode of ClusterGuardFL}
  \begin{algorithmic}[1]
    \Require {Input: $K$ clients chosen with a selection probability of $q_c \in (0, 1]$, client $k$'s local dataset $x_{i}^k$, loss function $L(\theta, x_{i})$}, learning rate $\eta$
    \Ensure Global model update $w_{t+1}$

    \State \textbf{Server executes:}
    \State Initialize global model $w_{t}$ randomly
    \For{$t \in \{1, 2, ..., T\}$}
      \For{each client $k \in \{1, 2, ..., K\}$}
        \State $w_{t+1}^k \gets \text{ClientTraining}(w_t, k)$
        \State $D_k \gets \text{ComputeDissimilarity}(w_t, w_{t+1}^k)$
      \EndFor
      
      \State \textbf{K-means clustering on dissimilarity scores:}
      \State $A_k, C \gets \text{KMeansClustering}(D_k \, \forall k \in K)$
      
      \State \textbf{Compute confidence scores for each client:}
      \State $S_k \gets \text{ConfidenceScore}(A_k, C, k) \, \forall k$
      
      \State \textbf{Model aggregation:}
      \State $w_{t+1} \gets \sum_{k=1}^{K} S_k \cdot (w_{t+1}^k - w_t)$
    \EndFor
    
    \hrulefill
    \State \textbf{ClientTraining}($w_t, k$):
    \State \hspace{1.5em} \textbf{for} $n \in \{1, 2, ..., N\}$ \textbf{do}
      \State \hspace{3em} Sample local batch $L_n$ with probability $q$
      \State \hspace{3em} \textbf{for} each sample $i \in L_n$ \textbf{do}
        \State \hspace{4.5em}\textbf{Compute gradient:} $g_n(x_i^k) \gets \nabla_{\theta_n} L(\theta_n, x_i^k)$
        \State \hspace{4.5em}\textbf{Update model:} $\theta_{n+1} \gets \theta_n - \eta g_n(x_i^k)$
      \State \hspace{3em} \textbf{end for}
    \State \hspace{1.5em} \textbf{end for}
    \State \hspace{1.5em} \textbf{Return $\theta_{n+1}$}
    
    \hrulefill
    \State \textbf{ComputeDissimilarity}($w_t, w_{t+1}^k$):
    \State \hspace{1.5em}\textbf{Compute predictions of global model: }
    \State \hspace{1.5em}$\hat{y}_t \gets f(w_t, x_i^k)$
    \State\hspace{1.5em}\textbf{Compute predictions of client model: }
    \State \hspace{1.5em}$\hat{y}_{t+1}^k \gets f(w_{t+1}^k, x_i^k)$
    \State \hspace{1.5em}\textbf{Compute EMD between $\hat{y}_t$ and $\hat{y}_{t+1}^k$: }
    \State \hspace{1.5em}$D_k \gets \text{EMD}(\hat{y}_t, \hat{y}_{t+1}^k)$
    \State \hspace{1.5em}\textbf{Return $D_k$}
    
    \hrulefill
    \State \textbf{KMeansClustering}($D_k, K$):
    \State \hspace{1.5em} \textbf{Initialize} centroids $C^{(0)} \gets \text{Random}(D_k, K)$
    \State \hspace{1.5em} \textbf{while} not converged \textbf{do}
      \State \hspace{3em}\textbf{Assign each $D_k$ to the closest centroid:} 
      \State \hspace{3em} $A_k \gets \arg\min_{j} \|D_k - C_j^{(\text{iter})}\|_2$
      \State \hspace{3em}\textbf{Update centroids:} $C_j^{(\text{iter}+1)} \gets \frac{1}{|A_j|} \sum_{k \in A_j} D_k$
    \State \hspace{1.5em} \textbf{end while}
    \State \hspace{1.5em} \textbf{Return $A_k$, $C$}
    \State \hrulefill
    \State \textbf{ConfidenceScore}($A_k, C, k$):
    \State \hspace{1.5em} \textbf{Compute distance from $D_k$ to its cluster centroid: }
    \State \hspace{1.5em} $d_k \gets \|D_k - C_j\|_2$, where $C_j$ is the centroid of $A_k$
    \State \hspace{1.5em} \textbf{Compute cluster size:} $|A_j|$
    \State \hspace{1.5em} \textbf{Compute confidence score:} 
    \State \hspace{1.5em} $S_k \gets \frac{|A_j|}{1 + d_k}$
    \State \hspace{1.5em} \textbf{Return} softmax-scaled $S_k$
  \end{algorithmic}
\end{algorithm}

\section{Convergence Analysis of the Proposed Algorithm}

In this section, we analyze the convergence properties of the proposed ClusterGuardFL algorithm. Our analysis examines the expected change in the loss function across consecutive aggregation steps. Similar assumptions have been utilized in prior works to analyze convergence in federated learning and distributed optimization settings \cite{nguyen2021federated, xing2021federated}. We start with the following set of assumptions:

\begin{enumerate}
    \item The function \( L(\theta) \) is convex.
    \item The gradient dissimilarity is bounded by a constant \( B \), such that \( \sum_{k=1}^{K} \|\nabla L_k(\theta) - \nabla L(\theta)\|^2 \leq B^2 \).
    \item \( L_k(\theta) \) satisfies Lipschitz continuity; specifically, there exists a constant \( M \) such that \( \|\nabla L_k(\theta) - \nabla L_k(\theta')\| \leq M \|\theta - \theta'\| \).
    \item The function \( L(\theta) \) satisfies the Polyak-Lojasiewicz (PL) inequality, i.e., \( \frac{1}{2} \|\nabla L(\theta)\|^2 \geq \mu (L(\theta) - L(\theta^*)) \), where \( \theta^* \) is the minimizer of \( L(\theta) \).
\end{enumerate}

The divergence observed in the gradients of local and aggregated loss functions is intricately linked to the distribution of data across diverse clients. To establish and elucidate this connection, our exploration commences with the introduction of the ensuing lemma.

\begin{lemma}
The dissimilarity between the local loss $L_k(\theta)$ and the global loss $L(\theta)$ can be bounded by the constant $U$ so the following holds true;
\begin{equation*}
    \sum_{k=1}^{K} \{||\triangledown L_k(\theta)||^2\} \le ||\triangledown L_(\theta)||^2 U^2.
\end{equation*}

\end{lemma}
\begin{proof}
From \textbf{Assumption 2} we have, 
\begin{equation}
\label{eq14}
    \sum_{k=1}^{K}\{||\triangledown L_k(\theta) - \triangledown L_(\theta) ||^2\} \le  B^2.
\end{equation}
Since we assume the loss function is convex and $\sum_{k=1}^{K}\{||\triangledown L_k(\theta)||\} = \triangledown L(\theta)$, \textbf{LHS} of [\ref{eq14}] can be elaborated as;

\begin{equation}
\begin{aligned}
&\sum_{k=1}^{K}\{||\triangledown L_k(\theta) - \triangledown L_(\theta) ||^2\}\\ =  &  \sum_{k=1}^{K}\{||\triangledown L_k(\theta)||^2\} -  2\sum_{k=1}^{K}\{||\triangledown L_k(\theta)||^T\}\triangledown L(\theta) + ||L(\theta)||^2 \\ 
= &\sum_{k=1}^{K}\{||\triangledown L_k(\theta)||^2\} - ||L(\theta)||^2.
\end{aligned}
\end{equation}

Therefore, [\ref{eq14}] can be re-written as;

\begin{equation}
\sum_{k=1}^{K} \{||\triangledown L_k(\theta)||^2\} \le ||\triangledown L(\theta)||^2 U^2,
\end{equation}

where, 
\begin{equation*}
U \ge \sqrt{1+\dfrac{B^2}{||\triangledown L_(\theta)||^2}} \ge 1.
\end{equation*}
\end{proof}

Now we investigate our algorithm’s convergence property. To begin, we present the following lemma to
derive an expected dissimilarity on the loss functions during each iteration.

\begin{lemma}\label{lem4}
The dissimilarity between the $t^{th}$ and $(t+1)^{th}$ weight updates at the global aggregation is upper bounded by;
     \begin{equation}
 \begin{aligned}
     L(\theta^{(t+1)}) -&L(\theta^t) \le   A_1||\triangledown L(\theta^{t})||^2  + \\ & A_2||\triangledown L(\theta^{t})\mathcal{N}(0, \sigma^{2}\mathrm{S}_fI)|| \\ & + A_3||\mathcal{N}(0, \sigma^{2}\mathrm{S}_fI)||^2
      \end{aligned}
 \end{equation}

where,

   $ A_1= \dfrac{U\gamma}{2}+ \dfrac{U\eta_t^2}{2\gamma q_c^2}+\dfrac{M\eta_t^2}{2 q_c^2}, A_2=\dfrac{\eta_t^2}{q_c}(\dfrac{U}{\gamma}+M)$ 
   
   and $A_3=\dfrac{\eta_{t}^2}{2}(\dfrac{U}{\gamma}+M).$
\end{lemma}
\begin{proof}

Let's consider the step where the server aggregates the client updates in FedAvg at $(t+1)^{th}$ round, 
    
\begin{equation}
\theta^{t+1} = \sum_{k=1}^K \dfrac{\eta_k}{q_c\eta}\theta_k^{(t+1)} +\mathcal{N}(0, \sigma^2).
\end{equation}
 Since $L(\theta)$ is Lipschitz continuous \textbf{(Assumption 3)}, 
 \begin{equation}
 \begin{aligned}
     \sum_{k=1}^{K}\{||L_k(\theta^{(t+1)})||\} \le & \sum_{k=1}^{K}\{||L_k(\theta^t)|| + \triangledown L_k(\theta^t)^T(\theta^{(t+1)} \\ & -\theta^t)\}  +\dfrac{M}{2}||\theta^{(t+1)}-\theta^t||^2.
      \end{aligned}
 \end{equation}
Since $\sum_{k=1}^{K}\{||\triangledown L_k(\theta^t)||\} =\triangledown L(\theta^t)U$ (lemma 3) and $\sum_{k=1}^{K}$ $\{||L_k(\theta^t)||\}$ = $L(\theta^t)$;

 \begin{equation}
 \label{eq2}
 \begin{aligned}
     L(\theta^{(t+1)})-L(\theta^t) \le & U\triangledown L(\theta^t)^T||\theta^{(t+1)} \\ & -\theta^t||\}  +\dfrac{M}{2}||\theta^{(t+1)}-\theta^t||^2.
      \end{aligned}
 \end{equation}

\begin{figure*}[!h]
\centering
\includegraphics[width=0.8\textwidth]{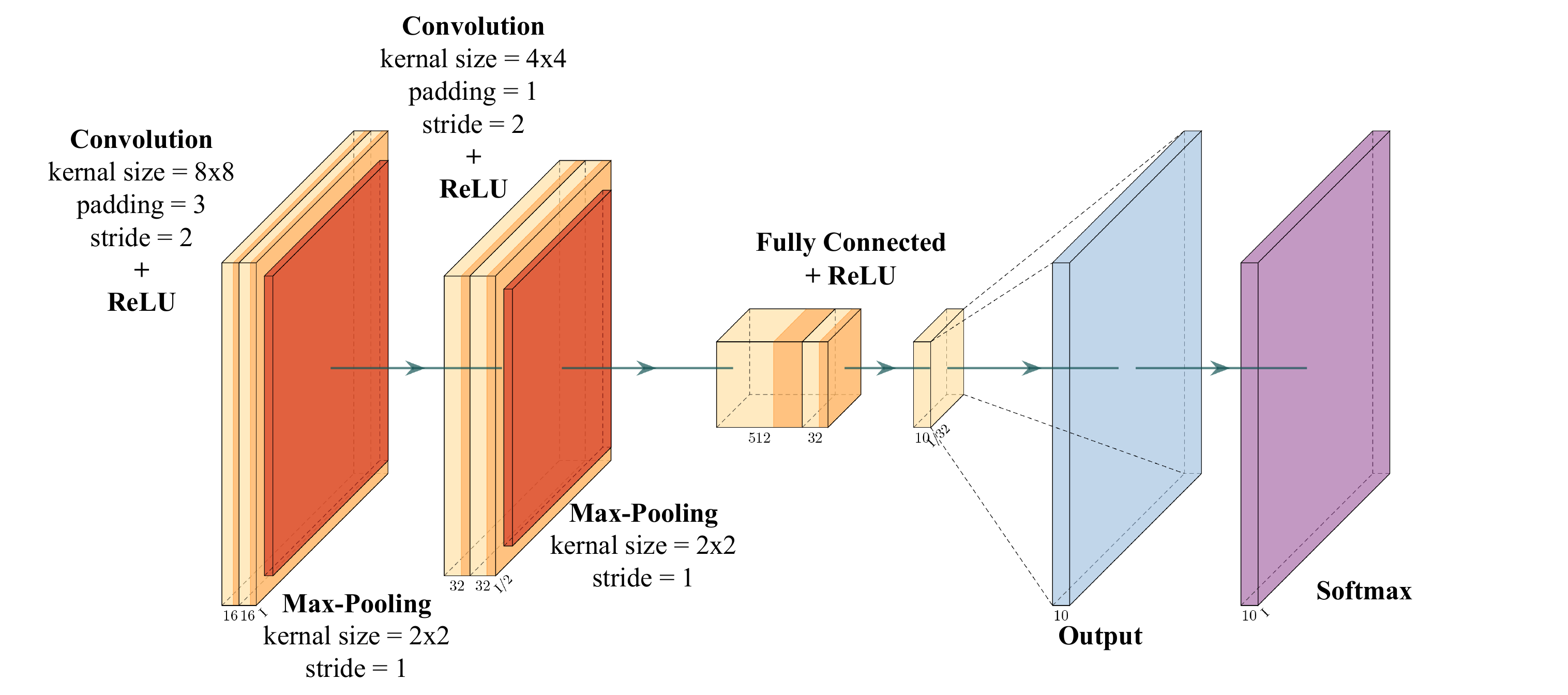}
\caption{Architecture of the Convolutional Neural Network (CNN) used to train on MNIST and Fashion MNIST datasets, consisting of multiple convolutional layers followed by pooling layers and fully connected layers.}
\label{mnist_cnn}
\end{figure*}

\begin{figure*}[!h]
\centering
\includegraphics[width=\textwidth]{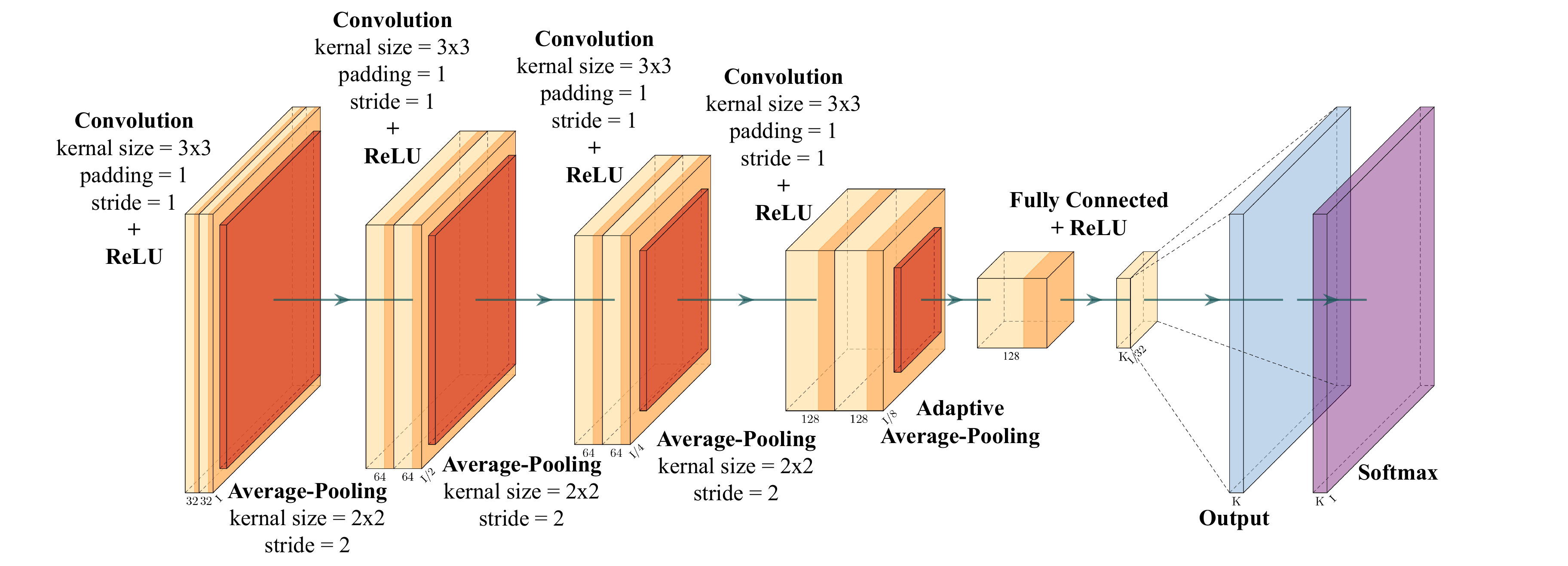}
\caption{Architecture of the Convolutional Neural Network (CNN) used to train on CIFAR10 dataset, consisting of multiple convolutional layers followed by pooling layers and fully connected layers.}
\label{cifar_cnn}
\end{figure*}

Recall that FedAvg updates Y as follows:
Now we must bound $||(\theta^{(t+1)} -\theta^t)||$. To do so, recall that FedAvg updates $\theta^{(t+1)}$ as follows;

 \begin{equation}
 \begin{aligned}
     \theta^{t+1} = \theta^{t}- \eta_{t}(\sum_{k=1}^K\dfrac{\eta_k}{q_c\eta}\triangledown  L_k(\theta^{t}) + \mathcal{N}(0, \sigma^{2}\mathrm{S}_fI)) \\
     ||\theta^{t+1} - \theta^{t}|| \le ||\dfrac{\eta_t}{q_c}\triangledown L(\theta^{t}) + \eta_{t}\mathcal{N}(0, \sigma^{2}\mathrm{S}_fI)||.
      \end{aligned}
 \end{equation}

Now, (\ref{eq2}) can be rewritten as;

 \begin{equation}
 \label{eq3}
 \begin{aligned}
     L(\theta^{(t+1)})-L(\theta^t) \le & U\triangledown L(\theta^t)^T||\dfrac{\eta_t}{q_c}\triangledown L(\theta^{t}) + \eta_{t}\mathcal{N}(0, \sigma^{2}\mathrm{S}_fI)|| \\ & +\dfrac{M}{2}||\dfrac{\eta_t}{q_c}\triangledown L(\theta^{t}) + \eta_{t}\mathcal{N}(0, \sigma^{2}\mathrm{S}_fI)||^2.
      \end{aligned}
 \end{equation}
 
By using Cauchy-Schwartz and Young's inequalities, for any $\gamma \in \mathbb{Z}^+$ $X^TY \le \dfrac{\gamma}{2}||X||^2 + \dfrac{1}{2\gamma}||Y||^2$ holds true. Therefore, (\ref{eq3}) yields, 
 
  \begin{equation*}
 \label{eq4}
 \begin{aligned}
     L(\theta^{(t+1)})-L(\theta^t) \le & \dfrac{U\gamma}{2}||\triangledown L(\theta^t)||^2 + \dfrac{U}{2\gamma}||\dfrac{\eta_t}{q_c}\triangledown L(\theta^{t}) \\ & + \eta_{t}\mathcal{N}(0, \sigma^{2}\mathrm{S}_fI)||^2    +\dfrac{M}{2}||\dfrac{\eta_t}{q_c}\triangledown L(\theta^{t}) \\ & + \eta_{t}\mathcal{N}(0, \sigma^{2}\mathrm{S}_fI)||^2,
      \end{aligned}
 \end{equation*}
 
   \begin{equation*}
 \begin{aligned}
     L(&\theta^{(t+1)})  -L(\theta^t) \le  \dfrac{U\gamma}{2}||\triangledown L(\theta^t)||^2 + \dfrac{U\eta_t^2}{2\gamma q_c^2}||\triangledown L(\theta^{t})||^2  + \\ & \dfrac{U\eta_{t}^2}{2\gamma}||\mathcal{N}(0, \sigma^{2}\mathrm{S}_fI)||^2 + \dfrac{U\eta_t^2}{\gamma q_c}||\triangledown L(\theta^{t})\mathcal{N}(0, \sigma^{2}\mathrm{S}_fI)||+ \\ & \dfrac{M\eta_t^2}{q_c}||\triangledown L(\theta^{t})\mathcal{N}(0, \sigma^{2}\mathrm{S}_fI)|| +  \dfrac{\eta_t^2M}{2 q_c^2}||\triangledown L(\theta^{t})||^2  + \\ & \dfrac{M\eta_{t}^2}{2}||\mathcal{N}(0, \sigma^{2}\mathrm{S}_fI)||^2,
    \end{aligned}
 \end{equation*}
 
    \begin{equation}
 \label{eq5}
 \begin{aligned}
     L(\theta^{(t+1)})  -L(\theta^t) \le &  (\dfrac{U\gamma}{2}+ \dfrac{U\eta_t^2}{2\gamma q_c^2}+\dfrac{M\eta_t^2}{2 q_c^2} )||\triangledown L(\theta^{t})||^2  + \\ & (\dfrac{U\eta_t^2}{\gamma q_c}+\dfrac{M\eta_t^2}{q_c})||\triangledown L(\theta^{t})\mathcal{N}(0, \sigma^{2}\mathrm{S}_fI)|| \\ & + (\dfrac{U\eta_{t}^2}{2\gamma}+\dfrac{M\eta_{t}^2}{2})||\mathcal{N}(0, \sigma^{2}\mathrm{S}_fI)||^2.
    \end{aligned}
 \end{equation}
 
 For ease of presentation, (\ref{eq5}) can be written as following, leading to the proof of lemma 4, 
 
     \begin{equation}
 \label{eq6}
 \begin{aligned}
     L(\theta^{(t+1)})  -L(\theta^t) \le &  A_1||\triangledown L(\theta^{t})||^2  + \\ & A_2||\triangledown L(\theta^{t})\mathcal{N}(0, \sigma^{2}\mathrm{S}_fI)|| \\ & + A_3||\mathcal{N}(0, \sigma^{2}\mathrm{S}_fI)||^2,
      \end{aligned}
 \end{equation}
where,

   $ A_1= \dfrac{U\gamma}{2}+ \dfrac{U\eta_t^2}{2\gamma q_c^2}+\dfrac{M\eta_t^2}{2 q_c^2}, A_2=\dfrac{\eta_t^2}{q_c}(\dfrac{U}{\gamma}+M)$ 
   
   and $A_3=\dfrac{\eta_{t}^2}{2}(\dfrac{U}{\gamma}+M).$

\end{proof}

The following lemma extends these results to upper bound the convergence of our algorithm for the optimal $\theta$ parameter.

\begin{lemma}
Given that parameter $\theta^*$ is the minimizer of $L(\theta)$, the convergence upper bound of the FL algorithm after T communication rounds is given by;

     \begin{equation*}
 \begin{aligned}
     L(\theta^{(T)}) - L(\theta^*)  \le &  (1+2A_1\mu)(L(\theta^0)-L(\theta^*))  + \\ & A_2||\triangledown L(\theta^{0})\mathcal{N}(0, \sigma^{2}\mathrm{S}_fI)|| + \\ &  A_3||\mathcal{N}(0, \sigma^{2}\mathrm{S}_fI)||^2.
      \end{aligned}
 \end{equation*}
 
 where, 

   $ A_1= \dfrac{U\gamma}{2}+ \dfrac{U\eta_t^2}{2\gamma q_c^2}+\dfrac{M\eta_t^2}{2 q_c^2}, A_2=\dfrac{\eta_t^2}{q_c}(\dfrac{U}{\gamma}+M)$ 
   
   and $A_3=\dfrac{\eta_{t}^2}{2}(\dfrac{U}{\gamma}+M).$

\end{lemma}

\begin{proof}
From lemma~\ref{lem4} and Assumption 4, we know that the following is true:

     \begin{equation}
 \label{eq7}
 \begin{aligned}
     L(\theta^{(t+1)})   \le &  L(\theta^t) + 2A_1\mu(L(\theta^t)-L(\theta^*))  + \\ & A_2||\triangledown L(\theta^{t})\mathcal{N}(0, \sigma^{2}\mathrm{S}_fI)|| \\ & + A_3||\mathcal{N}(0, \sigma^{2}\mathrm{S}_fI)||^2.
      \end{aligned}
 \end{equation}
 
 Subtracting $L(\theta^*)$ and applying recursively gives;

     \begin{equation}
 \label{eq9}
 \begin{aligned}
     L(\theta^{(T)}) - L(\theta^*)  \le &  (1+2A_1\mu)(L(\theta^0)-L(\theta^*))  + \\ & A_2||\triangledown L(\theta^{0})\mathcal{N}(0, \sigma^{2}\mathrm{S}_fI)|| + \\ &  A_3||\mathcal{N}(0, \sigma^{2}\mathrm{S}_fI)||^2.
      \end{aligned}
 \end{equation}
\end{proof}

To conclude, the convergence analysis of the proposed ClusterGuardFL algorithm provides a theoretical foundation for its robustness and effectiveness in federated learning scenarios. By leveraging the established assumptions, we derived bounds that capture the interplay between local client updates and global aggregation, accounting for the challenges posed by non-i.i.d. data distributions. 

\section{Experimental Evaluation}

In this section, we conduct a thorough evaluation of the proposed ClusterGuardFL framework using a diverse set of datasets to assess its effectiveness in mitigating data and model poisoning attacks while preserving fairness in FL.

\subsection{Datasets}

For our experimental evaluation, we employed three well-known datasets to encompass a range of data characteristics and complexities:

\subsubsection{MNIST}
The MNIST dataset consists of 28x28 grayscale images of handwritten digits (0-9). It serves as a fundamental benchmark in image classification tasks, offering simplicity and clarity in its patterns.

\subsubsection{Fashion MNIST}
Fashion MNIST is a dataset containing grayscale images of fashion items, each belonging to one of ten categories. This dataset provides a more challenging environment compared to MNIST, introducing variability in visual patterns and textures.

\subsubsection{CIFAR-10}
CIFAR-10 is a more complex dataset, comprising 32x32 color images across ten classes. It includes a diverse set of objects and scenes, making it suitable for evaluating the robustness and adaptability of our proposed ClusterGuardFL framework in scenarios with higher data intricacies.

The choice of these datasets enables a comprehensive assessment of ClusterGuardFL across varying levels of complexity, from simple hand-drawn digits to more intricate images of fashion items and objects. This diversity ensures that our evaluation captures the performance of ClusterGuardFL across a spectrum of FL applications.

\begin{table*}[t!]
\label{tab:iid}
\caption{Comparing the efficacy of state-of-the-art aggregation schemes against our proposed framework, "ClusterGuardFL," in i.i.d. settings under the influence of two common Byzantine attack schemes: label flipping and Gaussian noise. In all experimental settings, 20\% of the clients were deliberately set as malicious. The experiments were conducted over 100 communication rounds.}
\begin{tabular}{@{}cccccccccc@{}}
\toprule
\multirow{2}{*}{Algorithm} & \multicolumn{3}{c}{MNIST}             & \multicolumn{3}{c}{Fashion MNIST}     & \multicolumn{3}{c}{CIFAR10}           \\ \cmidrule(l){2-10} 
                           & No Attack & Label Flipping & Gaussian & No Attack & Label Flipping & Gaussian & No Attack & Label Flipping & Gaussian \\ \midrule
FedAvg                     & \textbf{94.97\%}   & 91.45\%        & 84.6\%   &\textbf{ 91.98\%}   & 83.85\%        & 80.7\%   & 64.97\%   & 51.42\%        & 49.40\%  \\
Median                     & 92.13\%   & 91.81\%        & 91.22\%  & 89.95\%   & 78.27\%        & 88.23\%  & 63.22\%   & 44.92\%        & 59.14\%  \\
Trimmed Mean               & 93.31\%   & 92.02\%        & 90.92\%  & 90.41\%   & 68.41\%        & 87.36\%  & 63.58\%   & 47.51\%        & 58.21\%  \\
Krum                       & 87.04\%   & 89.20\%        & 87.93\%  & 83.29\%   & 76.24\%        & 84.16\%  & 56.74\%   & 44.19\%        & 50.32\%  \\
GeoMed                     & 91.89\%   & 88.83\%        & 88.81\%  & 88.18\%   & 76.12\%        & 87.25\%  & 58.61\%   & 33.61\%        & 50.39\%  \\
AutoGM                     & 91.92\%   & 89.18\%        & 88.15\%  & 88.90\%   & 75.93\%        & 86.53\%  & 59.03\%   & 23.65\%        & 51.91\%  \\
Clustering                 & 94.93\%   & 91.90\%        & 84.22\%  & 91.12\%   & 86.49\%        & 83.48\%  & 64.87\%   & 49.23\%        & 50.18\%  \\
\textbf{ClusterGuardFL}            & 94.84\%   & \textbf{94.13\%}        & \textbf{93.60\%}  & 91.53\%   & \textbf{89.71\%}        & \textbf{90.44\%}  & \textbf{65.15\%}   & \textbf{58.93\%}        & \textbf{62.18\%}  \\ \bottomrule
\end{tabular}
\end{table*}

\begin{table*}[t!]
\caption{Comparing the efficacy of state-of-the-art aggregation schemes against our proposed framework, "ClusterGuardFL", in non i.i.d. settings  under the influence of two common Byzantine attack schemes: label flipping and Gaussian noise. In all experimental settings, 20\% of the clients were deliberately set as malicious. The experiments were conducted over 100 communication rounds.}
\label{tab:resultsnoniid}
\begin{tabular}{@{}cccccccccc@{}}
\toprule
\multirow{2}{*}{Algorithm} & \multicolumn{3}{c}{MNIST}             & \multicolumn{3}{c}{Fashion MNIST}     & \multicolumn{3}{c}{CIFAR10}           \\ \cmidrule(l){2-10} 
                           & No Attack & Label Flipping & Gaussian & No Attack & Label Flipping & Gaussian & No Attack & Label Flipping & Gaussian \\ \midrule
FedAvg                     & 90.20\%   & 72.46\%        & 69.12\%  & \textbf{87.13\%}   & 86.26\%        & 29.20\%  & \textbf{61.79\%}   & 49.85\%        & 43.69\%  \\
Median                     & 73.45\%   & 53.95\%        & 73.00\%  & 78.57\%   & 74.93\%        & 80.23\%  & 56.05\%   & 42.73\%        & 40.71\%  \\
Trimmed Mean               & \textbf{90.36\% }  & 70.88\%        & 82.98\%  & 84.32\%   & 76.75\%        & 84.49\%  & 61.47\%   & 45.94\%        & 37.63\%  \\
Krum                       & 72.16\%   & 76.01\%        & 74.84\%  & 85.51\%   & 86.12\%        & \textbf{86.54\%}  & 25.24\%   & 25.28\%        & 19.44\%  \\
GeoMed                     & 50.93\%   & 51.58\%        & 44.01\%  & 73.74\%   & 75.52\%        & 76.06\%  & 25.45\%   & 22.11\%        & 24.25\%  \\
AutoGM                     & 51.68\%   & 53.58\%        & 43.91\%  & 71.36\%   & 73.39\%        & 65.94\%  & 28.52\%   & 18.01\%        & 26.82\%  \\
Clustering                 & 89.94\%   & 83.44\%        & 81.28\%  & 86.15\%   & 85.74\%        & 67.22\%  & 59.87\%   & 46.96\%        & 39.80\%  \\
\textbf{ClusterGuardFL}              & 90.28\%        & \textbf{89.11\%}             & \textbf{88.98\%}      & 87.06\%        & \textbf{86.83\%}             & 84.41\%       & 61.10\%        & \textbf{57.87\% }            & \textbf{57.14\%}       \\ \bottomrule
\end{tabular}
\end{table*}

\subsection{Experimental Setup}
All experiments presented in this section were implemented in Python using the PyTorch framework, running on a computing environment equipped with an Intel Core i7 11th generation processor. Model training and inference tasks were accelerated using an NVIDIA RTX 3080 GPU, which provided parallel processing capabilities to expedite the computational workload.

For the MNIST and Fashion MNIST datasets, we utilized a Convolutional Neural Network (CNN) architecture as illustrated in Fig.2. This architecture consists of multiple convolutional layers, followed by pooling layers and fully connected layers, designed to effectively capture the patterns in the image data.

For the CIFAR10 dataset, a more complex CNN architecture was employed, which is shown in Fig.3. This architecture includes several convolutional layers, each followed by pooling and activation functions, to handle the higher complexity of the CIFAR10 dataset.

The choice of architectures was guided by the need for a balance between performance and computational efficiency, ensuring that the models are robust enough for the datasets while maintaining manageable training times.

\subsection{Discussion}

In this section, we evaluate the performance of our proposed framework, ClusterGuardFL, against several state-of-the-art (SoTA) aggregation schemes. The experiments were conducted in both Independent and Identically Distributed (I.I.D.) and non-I.I.D. settings, under the influence of two common Byzantine attack schemes: label flipping and Gaussian noise. In all experimental settings, 20\% of the clients were deliberately set as malicious, and the experiments were conducted over 100 communication rounds.

Table II presents the comparative analysis of various aggregation schemes, including FedAvg, Median, Trimmed Mean, Krum, GeoMed, AutoGM, Clustering, and our proposed ClusterGuardFL, in I.I.D. settings. For the MNIST dataset, ClusterGuardFL improved accuracy by 2.11\% over the best performing SoTA, Trimmed Mean, under label flipping and by 2.38\% over Median under Gaussian noise. In the Fashion MNIST dataset, ClusterGuardFL showed a 3.22\% improvement over best-performing SoTA, clustering, under label flipping, and a 2.21\% improvement over median under gaussian noise. For the CIFAR10 dataset, ClusterGuardFL achieved a notable 7,51\% improvement over FedAvg under label flipping and a 3.04\% improvement over median under gaussian noise. Also, there was a slight increase of 0.18\% in accuracy for CIFAR10 dataset under no attack compared to FedAvg.  These results demonstrate that ClusterGuardFL consistently outperforms existing state-of-the-art aggregation schemes, particularly under adversarial conditions, validating its efficacy in maintaining high accuracy and resilience in FL environments with Byzantine attacks.

However, it is important to note that in some scenarios, ClusterGuardFL does not outperform certain state-of-the-art schemes. For instance, in the MNIST dataset under no attack, FedAvg performed slightly better than ClusterGuardFL, achieving 94.97\% accuracy compared to ClusterGuardFL's 94.84\%. Additionally, for the Fashion MNIST dataset under no attack in the I.I.D. setting, FedAvg again outperforms ClusterGuardFL by 0.45\% (91.98\% vs. 91.53\%), demonstrating that the efficacy of ClusterGuardFL may vary depending on the nature of the attack and dataset characteristics. Despite these specific cases, ClusterGuardFL demonstrates superior performance overall under adversarial conditions, validating its robustness in maintaining high accuracy and resilience in DL environments with Byzantine attacks.

Now, we evaluate the performance of our proposed framework, ClusterGuardFL, against several state-of-the-art aggregation schemes in non-I.I.D. settings under the influence of the same Byzantine attack schemes. Table III presents the comparative analysis across three datasets. The MNIST dataset, ClusterGuardFL improved accuracy by 5.67\% over clustering under label flipping and by 6.01\% over trimemd mean under gaussian noise. In the Fashion MNIST dataset, ClusterGuardFL showed a 0.57\% improvement over FedAvg under label flipping. For the CIFAR10 dataset, ClusterGuardFL achieved a notable 8.02\% improvement over FedAvg under label flipping and a 13.45\% improvement over FedAvg under Gaussian noise. 

Nevertheless, there are cases where ClusterGuardFL does not outperform certain SoTA approaches. For example, in the non I.I.D. setting for the MNIST dataset under no attack, \textbf{Trimmed Mean} outperformed ClusterGuardFL by a margin of 0.08\% (90.36\% vs. 90.28\%), indicating that even under non-adversarial conditions, certain aggregation schemes may be more effective. Similarly, in the Fashion MNIST dataset under no attack, \textbf{FedAvg} demonstrated slightly better performance than ClusterGuardFL, achieving 87.13\% accuracy compared to ClusterGuardFL's 87.06\%. In the CIFAR10 dataset, \textbf{FedAvg} also outperformed ClusterGuardFL under no attack, achieving a margin of 0.69\% (61.79\% vs. 61.10\%).

Additionally, under the Gaussian noise attack in the Fashion MNIST dataset, \textbf{Krum} outperformed ClusterGuardFL by 2.13\% (86.54\% vs. 84.41\%), highlighting that in some cases, other techniques may handle specific attacks more effectively.

These results emphasize that while ClusterGuardFL consistently outperforms existing state-of-the-art aggregation schemes under adversarial conditions, there are particular scenarios where other approaches demonstrate higher efficacy. Nonetheless, ClusterGuardFL remains highly effective in maintaining robust performance and resilience in federated learning environments, especially in the presence of Byzantine attacks.

\section{Conclusion}

In this paper, we introduced "ClusterGuardFL," a novel aggregation framework designed to enhance the robustness and accuracy of FL systems against Byzantine attacks. Our method utilizes secure cluster-weighted client aggregation to mitigate the impact of malicious clients, ensuring reliable model training.

Extensive experiments on MNIST, Fashion MNIST, and CIFAR10 datasets in both i.i.d. and non-i.i.d. settings demonstrated that ClusterGuardFL consistently outperforms state-of-the-art aggregation schemes under label flipping and Gaussian noise attacks. These results underscore the robustness and efficacy of ClusterGuardFL in maintaining high accuracy even with 20\% malicious clients. Our findings suggest that ClusterGuardFL is a promising solution for enhancing FL systems' resilience. Future work will explore integrating ClusterGuardFL with other security mechanisms to further boost its robustness and scalability.

\bibliographystyle{IEEEtran}
\bibliography{Main}

\vfill

\end{document}